\documentclass[10pt]{article} 

\usepackage{pdfpages,fullpage,times}
\usepackage{booktabs}       

\usepackage{graphicx,subcaption,color,tikz}

\usepackage{algorithm,algorithmic,hyperref}

\usepackage{amsthm}
\usepackage{Definitions}
\usepackage{xspace,enumitem}

\newtheorem*{rep@theorem}{\rep@title}

\usepackage{thmtools,thm-restate}

\usepackage{color}

\newcommand{\Bo}[1]{}
\newcommand{\yingyu}[1]{}

\newcommand{\eigg}[1]{\lambda_{#1}(G)} 
\newcommand{\eiggn}[1]{\lambda_{#1}(G_n)} 
\newcommand{\constg}{c_g} 

\title{Diverse Neural Network Learns True Target Functions}

\author{
Bo Xie\thanks{College of Computing, Georgia Institute of Technology. Email:\texttt{bo.xie@gatech.edu}}, ~
Yingyu Liang\thanks{Department of Computer Science, Princeton University. Email: \texttt{yingyul@cs.princeton.edu}},~
Le Song\thanks{College of Computing, Georgia Institute of Technology. Email:\texttt{lsong@cc.gatech.edu}}
}
\date{}

\begin{document}

\maketitle

\begin{abstract}
Neural networks are a powerful class of functions that can be trained with simple gradient descent to achieve state-of-the-art performance on a variety of applications. Despite their practical success, there is a paucity of results that provide theoretical guarantees on why they are so effective. 
Lying in the center of the problem is the difficulty of analyzing the non-convex loss function with potentially numerous local minima and saddle points. Can neural networks corresponding to the stationary points of the loss function learn the true target function? If yes, what are the key factors contributing to such nice optimization properties? 

In this paper, we answer these questions by analyzing one-hidden-layer neural networks with ReLU activation, and show that despite the non-convexity, neural networks with diverse units have no spurious local minima. We bypass the non-convexity issue by directly analyzing the first order optimality condition, and show that the loss can be made arbitrarily small if the minimum singular value of the ``extended feature matrix'' is large enough. We make novel use of techniques from kernel methods and geometric discrepancy, and identify a new relation linking the smallest singular value to the spectrum of a kernel function associated with the activation function and to the diversity of the units. Our results also suggest a novel regularization function to promote unit diversity for potentially better generalization.
\end{abstract}

\vspace{-3mm}
\section{Introduction}
\vspace{-2mm}

\setlength{\abovedisplayskip}{4pt}
\setlength{\abovedisplayshortskip}{1pt}
\setlength{\belowdisplayskip}{4pt}
\setlength{\belowdisplayshortskip}{1pt}
\setlength{\jot}{3pt}


Neural networks are a powerful class of nonlinear functions which have been successfully deployed in a variety of machine learning tasks. In the simplest form, neural networks with one hidden layer are linear combinations of nonlinear basis functions (units), 
\begin{equation}
  \label{eq:neuralnet}
  f(x) = \sum_{k=1}^n v_k \sigma(w_k^\top x) 
\end{equation}
where $\sigma(w_k^\top x)$ is a basis function with weights $w_k$, and $v_k$ is the corresponding combination coefficient. Learning with neural networks involves adapting both the combination coefficients and the basis functions at the same time, usually by minimizing the empirical loss
\begin{equation}
  \label{eq:emploss}
  L(f) = \frac{1}{2m} \sum_{l=1}^m \ell(y_l, f(x_l))
\end{equation}
with first-order methods such as (stochastic) gradient descent. It is believed that 
basis function adaptation is a crucial ingredient for neural networks to achieve more compact models and better performance~\cite{Barron93,YanMocDenFreetal14}. 

However, the empirical loss minimization problem involved in neural network training is non-convex with potentially numerous local mini\-ma and saddle points. This makes formal analysis of training neural networks very challenging. Given the empirical success of neural networks, a sequence of important and urgent scientific questions need to be investigated: Can neural networks corresponding to stationary points of the empirical loss learn the true target function? If the answer is yes, then what are the key factors contributing to their nice optimization properties? Based on these understandings, can we design better regularization schemes and learning algorithms to improve the training of neural networks?    


In this paper, we provide partial answers to these questions by analyzing one-hidden-layer neural networks with rectified linear units (ReLU) in a least-squares regression setting. We show that neural networks with diverse units have no spurious local minima.
More specifically, we show that the training loss of neural networks decreases in proportion to $\nbr{\partial L/\partial W}^2/s_m^2(D)$ where $\partial L/\partial W$ is the gradient and $s_m(D)$ is the minimum singular value of the extended feature matrix $D$ (defined in Section~\ref{sec:first_order}). The minimum singular value is lower bounded by two terms, where the first term is related to the spectrum property of the kernel function associate with the activation $\sigma(\cdot)$, and the second term quantifies the diversity of the units, measured by the classical notion of geometric discrepancy of a set of vectors. Essentially, the slower the decay of the spectrum, the better the optimization landscape; the more diverse the unit 
weights, the more likely stationary points will result in small training loss and generalization error. 

We bypass the hurdle of non-convexity by directly analyzing the first order optimality condition of the learning problem, which implies that there are no spurious local minima if the minimum singular value of the extended feature matrix is large enough. 
Bounding the singular value is challenging because it entangles the nonlinear activation function, the weights and data in a complicated way. Unlike most previous attempts, we directly analyze the effect of nonlinearity without assuming independence of the activation patterns from actual data; in fact, the dependence of the patterns on the data and the unit weights underlies the key connection to activation kernel spectrum and the diversity of the units.

We have constructed a novel proof, which makes use of techniques from geometric discrepancy and kernel methods, and have identified a new relation linking the smallest singular value to the diversity of the units and the spectrum of a kernel function associated with the unit. More specifically, 
%
%
\begin{itemize}[leftmargin=*,nosep,nolistsep]
  \item We identify and separate two factors in the minimum singular value: 1) an ideal spectrum that is related to the kernel of the activation function and an ideal configuration of diverse unit weights; 2) deviation from the ideal spectrum measured by how far away actual unit weights are from the diverse configuration. This new perspective reveals benign conditions in learning neural networks.
  \item We characterize the deviation from the ideal diverse weight configuration using the concept of discrepancy, which has been extensively studied in the geometric discrepancy theory. This reveals an interesting connection between the discrepancy of the weights and the training loss of neural networks. Therefore, it serves as a clean tool to analyze and verify the learning and the generalization ability of the networks.
\end{itemize}
Our results also suggest a novel regularization scheme to promote unit diversity for potentially better generalization. In~\cite{MarSra15}, it is shown that diversity of the neurons leads to smaller network size and better performance.

Whenever possible, we corroborate our theoretical analysis with numerical simulations. These numerical results include computing and verifying the relationship between the discrepancy of a learned neural network and the minimum singular value. Additionally, we measure the effects on the discrepancy with and without regularization. In all these examples, the experiments match with the theory nicely and they accord with the practice of using gradient descent to learn neural networks.

\vspace{-3mm}
\section{Related work}
\vspace{-2mm}

Kernel methods have many commonalities with one-hidden-layer neural networks. The random feature perspective~\cite{RahRec09,ChoSaul09} views kernels as linear combinations of nonlinear basis functions, similar to neural networks. The difference between the two is that the weights are random in kernels while in neural networks they are learned.
Using learned weights leads to considerable smaller models as shown in ~\cite{Barron93}. However it is a non-convex problem and it is difficult to find the global optima. \eg, one-hidden-layer networks are NP-complete to learn in the worst case~\cite{BluRiv93}. We will make novel use of techniques from kernel methods to analyze learning in neural networks.  

The empirical success of training neural networks with simple algorithms such as gradient descent has motivated researchers to explain their surprising effectiveness. In~\cite{ChoHenMatAroLec15}, the authors analyze the loss surface of a special random neural network through spin-glass theory and show that for many large-size networks, there is a band of exponentially many local optima, whose loss is small and close to that of a global optimum. The analyzed polynomial network is different from the actual neural network being used which typically contains ReLU nowadays. Moreover, the analysis does not lead to a generalization guarantee for the learned neural network.

A similar work shows that all local optima are also global optima in linear neural networks \cite{Kawaguchi16}. However their analysis for nonlinear neural networks hinges on independence of the activation patterns from the actual data, which is unrealistic. Some other works try to argue that gradient descent is not trapped in saddle points~\cite{LeeSimJorRec16,GeHuaJinYua15}, as was suggested to be the major obstacle in optimization~\cite{DauPasGulChoetal14}. There is also a seminal work using tensor method to avoid the non-convex optimization problem in neural network~\cite{MajSedAna15}. However, the resulting algorithm is very different from typically used algorithms where only gradient information of the empirical loss $L$ is used. 

\cite{SouCar16} is the closest to our work, which shows that zero gradient implies zero loss for all weights except an exception set of measure zero. However, this is insufficient to guarantee low training loss since small gradient can still lead to large loss. 
\yingyu{the smooth analysis depends on perturbation noise, but this is not the same as dropout. }
Furthermore, their analysis does not characterize the exception set and it is unclear a priori whether the set of local minima fall into the exception set.


\vspace{-3mm}
\section{Problem setting and preliminaries}
\vspace{-2mm}


We will focus on a special class of data distributions where the input $x \in \RR^d$ is drawn uniformly from the unit sphere,\footnote{The restriction of input to the unit sphere is not too stringent since many inputs are already normalized. Furthermore, it is possible to relax the uniform distribution to sub-gaussian rotationally invariant distributions, but we use the current assumption for simplicity.} 
and assume that $\abr{y} \le Y$.
We consider the following hypothesis class:
\begin{align}
  \Fcal = \cbr{ \sum_{k=1}^n v_k \sigma(w_k^\top x) :  v_k \in \cbr{\pm 1},  \sum_{k=1}^n \nbr{w_k} \le C_W }  
\end{align}
where $\sigma(\cdot)=\max\{0, \cdot\}$ is the rectified linear unit (ReLU) activation function, $\cbr{w_k}$ and $\cbr{v_k}$ are the unit weights and combination coefficients respectively, $n$ is the number of units, and $C_W$ is some constant. We restrict $v_k \in \cbr{-1, 1}$ due to the positive homogeneity of ReLU, 
\begin{align*}
	f(x) = \sum_{k=1}^n v_k \sigma(w_k^\top x) = \sum_{k=1}^n \frac{v_k}{|v_k|} \sigma(|v_k| w_k^\top x).  
\end{align*}
That is, the magnitude of $v_k$ can always be scaled into the corresponding $w_k$. 
For convenience, let 
\begin{align}
   W:=(w_1^\top,w_2^\top,\ldots,w_k^\top)^\top
\end{align}
be the column concatenation of the unit parameters; also let $W$ denote the set $\cbr{w_i: 1\le i \le k}$. Let 
\begin{align}
 \Fcal_W = \cbr{W: \sum_k \nbr{w_k} \le C_W}
\end{align} 
denote the feasible set of $W$'s. A function $f \in \Fcal$ will depend on $v$ and $W$, and it can be written as $f(x; v, W)$. But when clear from the context, it is shorten as $f(x)$.

Given a set of \iid~training examples $\cbr{x_l, y_l}_{l=1}^m \subseteq \RR^d \times \RR$, we want to find a function $f\in\Fcal$ which minimizes the following least-squares loss function
\footnote{Our approach can also be applied to some other loss functions. For example, in logistic regression for $y\in \cbr{\pm 1}$, when the loss is $\frac{-1}{m} \sum_{\ell=1}^m \log \sbr{\text{sigmoid}(y_\ell f(x_\ell))}$, we can bound the error $\EE\sbr{\text{sigmoid}(y f(x)) - 1}^2$. 
}
\begin{align*}
L(f) =  \frac{1}{2m}\sum_{l=1}^m\rbr{y_l - f(x_l)}^2 .
\end{align*}
Typically, gradient descent over $L(f)$ is used to learn all the parameters in $f$, and a solution with small gradient is returned at the end.\footnote{Note that even though ReLU is not differentiable, we can still use its sub-gradient by defining $\sigma'(u) = \II\sbr{u > 0}$.} However, adjusting the bases $\cbr{w_k}$ leads to a non-convex optimization problem, and there is no theoretical guarantee that gradient descent can find global optima.

Our primary goal is to identify conditions under which there are no spurious local minima.
 We need to identify a set $\Gcal_W \subseteq \Fcal_W$ such that when gradient descent outputs a solution $W \in \Gcal_W$ with the gradient norm $\nbr{\partial L / \partial W}$ smaller than $\epsilon$, then the training and test errors can be bounded by $O(\epsilon^2)$. Ideally, $\Gcal_W$ should have clear characterization that can be easily verified, and should contain most $W$ in the parameter space (especially those solutions obtained in practice). 

On notation, we will use $c$, $c'$ or $C$, $C'$ to denote constants and its value may change from line to line.

\subsection{First order condition} \label{sec:first_order}

In this section, we will rewrite the set of first order conditions for minimizing the empirical loss $L$. This rewriting motivates the direction of our later analysis. More specifically, 
the gradient of the empirical loss w.r.t.\ $w_k$ is
\begin{align} \label{eqn:derivative}
\frac{\partial L}{\partial w_k} &= \frac{1}{m} \sum_{l=1}^m \rbr{f(x_l) - y_l}v_k\sigma^{\prime}(w_k^\top x_l) x_l, 
\end{align}
for all $k=1,\ldots,n$. We will express this collection of gradient equations using matrix notation. Define the ``extended feature matrix'' as 
\begin{align}
D &= \rbr{\begin{array}{ccc}
v_1\sigma^{\prime}(w_1^\top x_1) x_1 & \cdots & v_1\sigma^{\prime}(w_1^\top x_m) x_m \\
\cdots & \cdots  & \cdots \\
v_k\sigma^{\prime}(w_k^\top x_1) x_1 & \cdots & v_k\sigma^{\prime}(w_k^\top x_m) x_m \\
\cdots & \cdots  & \cdots \\
v_n\sigma^{\prime}(w_n^\top x_1) x_1 & \cdots & v_n\sigma^{\prime}(w_n^\top x_m) x_m
\end{array}}, \nonumber
\end{align}
and the residual as 
$$r = \frac{1}{m}\rbr{f(x_1)-y_1, \cdots, f(x_m) - y_m}^\top.$$
Then we have
\begin{align}
\frac{\partial L}{\partial W} := \rbr{\frac{\partial L}{\partial w_1}^\top,\ldots,\frac{\partial L}{\partial w_n}^\top}^\top = D\ r. 
\end{align}
A stationary point has zero gradient, so if $D \in \RR^{dn \times m}$ has full column rank, then immediately $r= 0$, \ie, it is actually a global optimal. Since $nd > m$ is necessary for $D$ to have full column rank, we assume this throughout the paper. 

However, in practice, we will not have the gradient being exactly zero because, \eg, we stop the algorithm in finite steps or because we use stochastic gradient descent (SGD). In other words, typically we only have $\nbr{\partial L / \partial W}\leq \epsilon$, and $D$ being full rank is insufficient since small gradient can still lead to large loss. More specifically, let $s_m(D)$ be the minimum singular value of $D$, we have
\begin{align}\label{eqn:residual_bound}
  \nbr{r} \le \frac{1}{s_m(D)}\nbr{\frac{\partial L}{\partial W}}.  
\end{align}
We can see that $s_m(D)$ needs to be large enough for the residual to be small. Thus it is important to identify conditions to lower bound $s_m(D)$ away from zero, which will be the focus of the paper. 

\vspace{-3mm}
\subsection{Spectrum decay of activation kernel}
\vspace{-2mm}

We will later show that $s_m(D)$ is related to the decay rate of the kernel spectrum associated with the activation function. More specifically, for an activation function $\sigma(w^\top x)$, we can define the following kernel function
\begin{align}\label{eqn:gxy}
g(x,y) = \EE_w [\sigma'(w^\top x) \sigma'(w^\top y) \inner{x}{y}]
\end{align}
where $\EE_w$ is over $w$ uniformly distributed on a sphere. 

In particular, for ReLU, the kernel has a closed form
\begin{align}\label{eqn:relu_kernel}
g(x,y) = \rbr{\frac{1}{2} - \frac{\arccos\inner{x}{y}}{2\pi}}  \inner{x}{y}. 
\end{align} 
In fact, it is a dot-product kernel and its spectrum can be obtained through spherical harmonic decomposition:
\begin{align}
  g(x,y) = \sum_{u=1}^\infty \gamma_u \phi_u(x) \phi_u(y),
\end{align}
where the eigenvalues are ordered $\gamma_1 \ge \cdots \ge \gamma_m \ge \cdots \ge 0$ and the bases $\phi_u(x)$ are spherical harmonics. The $m$-th eigenvalue $\gamma_m$ will be related to $s_m(D)$.

For each spherical harmonic of order $t$, there are $N(d, t)= \frac{2t+d-2}{t}\mychoose{t+d-3}{t-1}$ basis functions sharing the same eigenvalue. Therefore, the spectrum has a step like shape where each step is of length $N(d,t)$. Especially, for high dimensional input $x$, the number of such basis functions with large eigenvalues can be very large. Figure~\ref{fig:spectrum} illustrates the spectrum of the kernel for $d=1500$, and it is about $\Omega(m^{-1})$ for a large range of $m$.
For more details about the decomposition, please refer to Appendix~\ref{app:spherical_spectrum}.

Such step like shape also appears in the Gram matrix associated with the kernel. Figure~\ref{fig:spectrum_approx} compares the spectra of the kernel of $d=15$ and the corresponding Gram matrix with $m=3000$. We can see the spectrum of the Gram matrix closely resembles that of the kernel. Such concentration phenomenon underlies the reason why the spectrum of $D^\top D$ is closely related to the corresponding kernel.

\begin{figure}
\centering
\begin{minipage}[b]{0.48\textwidth}
\centering
\includegraphics[width=0.8\columnwidth]{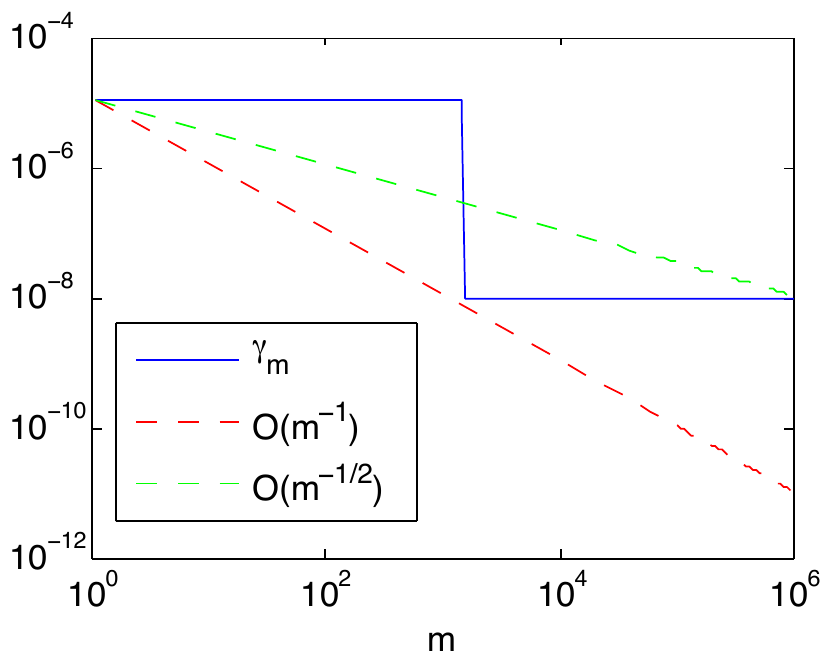}
\caption{The spectrum decay of the kernel associated with ReLU. We set $d = 1500$. It is decays slower than $O(1/m)$ for a large range of $m$.}\label{fig:spectrum}
\end{minipage}
\hfill
\begin{minipage}[b]{0.48\textwidth}
\centering
\includegraphics[width=0.8\columnwidth]{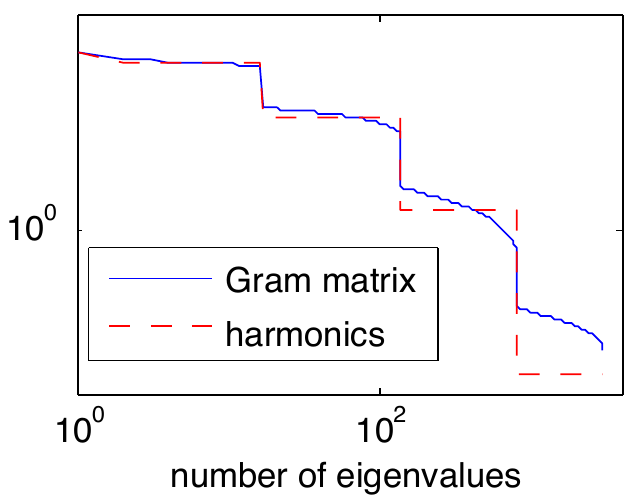}
\caption{The spectrum of a Gram matrix concentrates around the spherical harmonic spectrum of the kernel.}\label{fig:spectrum_approx}
\end{minipage}
\end{figure}



\vspace{-3mm}
\subsection{Weight discrepancy} 
\vspace{-2mm}

Another key factor in the analysis is the diversity of the unit weights, measured by the classic notion of geometric discrepancy~\cite{BilLac15}. Given a set of $n$ points $W = \cbr{w_k}_{k=1}^n$ on the unit sphere $\mathbb{S}^{d-1}$,
the discrepancy of $W$ w.r.t.\ a measurable set $S \subseteq \mathbb{S}^{d-1}$ is defined as
\begin{align}
  \text{dsp}(W, S) = \frac{1}{n} \abr{W \cap S} - A(S),
\end{align}
where $A(S)$ is the normalized area of $S$ (i.e., the area of the whole sphere $A(\mathbb{S}^{d-1}) = 1$). $\text{dsp}(W, S)$ quantifies the difference between the empirical measure of $S$ induced by $W$ and the measure of $S$ induced by a uniform distribution.  

By defining a collection $\Scal$ of such sets, we can summarize the difference in the empirical measure induced by $W$ versus the uniform distribution over the sphere. More specifically, we will focus on the set of slices, each defined by a pair of inputs $x,y \in \mathbb{S}^{d-1}$, \ie, 
\begin{align}
  \Scal &= \cbr{S_{xy} : x, y \in \mathbb{S}^{d-1}},~\textnormal{where}\nonumber\\
  S_{xy} &= \cbr{w \in \mathbb{S}^{d-1} : w^\top x \ge 0, w^\top y \ge 0}. \label{eqn:slice}
\end{align}
Essentially, each $S_{xy}$ defines a slice-shaped area on the sphere which is carved out by the two half spaces $w^\top x\geq 0$ and $w^\top y\geq 0$. 

Based on the collection $\Scal$, we can define two discrepancy measures relevant to ReLU units. 
$L_\infty$ discrepancy of $W$ w.r.t.\ $\Scal$ is defined as
\begin{align}
  \label{eq:linf_dsp}
  L_\infty(W, \Scal) = \sup_{S \in \Scal} \abr{\text{dsp}(W, S)},
\end{align}
and the $L_2$ discrepancy as
\begin{align}
  \label{eq:l2_dsp}
  L_2(W, \Scal) = \sqrt{ \EE_{x,y} \text{dsp}(W, S_{xy})^2}
\end{align}
where the expectation is taken over $x,y$ uniformly on the sphere. 
We use $L_\infty(W)$ and $L_2(W)$ as their shorthands. Both discrepancies measure how diverse the points $W$ are. The more diverse the points, the smaller the discrepancy.

For our analysis, we slightly generalize the discrepancy for $w_k$'s not on unit sphere, by setting
\begin{align}
 \text{dsp}(\cbr{w_k}_k, S) = \text{dsp}(\cbr{ w_k / \|w_k\|}_k, S).
\end{align}

\vspace{-3mm}
\section{Main results}
\vspace{-2mm}

Our main result is a bound on the training loss and the generalization error, assuming sufficiently large $n, d$. 
To state the theorem, first recall that $\beta \in (0, 1)$ is the decay exponent of the spectrum of the activation kernel in \eq{eqn:relu_kernel}, that is, $\gamma_m$ is the $m$-th eigenvalue of the kernel and satisfies $\gamma_m = \Omega(m^{-\beta})$.
Also recall that $\Fcal_W$ denote the set of feasible values of $W$.

\begin{theorem}[Main, simplified] \label{thm:main}
Let $\delta, \delta' \in (0, 1)$. If 
\[
  n = \tilde\Omega(m^{\beta}), \quad d = \tilde\Omega(m^{\beta}),
\]
then there exists a set $\Gcal_W \subseteq \Fcal_W$ which takes up $1-\delta'$ fraction of measure of $\Fcal_W$, such that with probability at least $1- cm^{-\log m} - \delta$ the following holds. For any $W \in \Gcal_W$ and any $v \in \cbr{-1, 1}^n$, we have 
\begin{align*}
  & \frac{1}{2m} \sum_{\ell=1}^m (f(x_\ell; v,W) - y_\ell)^2  \le c \nbr{\frac{\partial L}{\partial W}}^2, \\
	& \frac{1}{2} \EE (f(x; v,W) - y)^2 \le  c\nbr{\frac{\partial L}{\partial W}}^2 + c'(C_W + Y)^2 \sqrt{\frac{1}{m}\log \frac{1}{\delta}}.
\end{align*}
\end{theorem}
Here $\tilde\Omega$ hides logarithmic terms $\log m \log \frac{1}{\delta} \log\frac{1}{\delta'}$.

\paragraph{Remark 1.}  Intuitively, the theorem means that when $n, d$ are sufficiently large, for most feasible weights $W$, the training loss is in the order of the square norm of the gradient, and the generalization error has an additional term $\Otil(1/\sqrt{m})$. 
In particular, when we obtain an solution $W \in \Gcal_W$ with gradient $\nbr{\partial L/\partial W}^2 \le \epsilon$, the training loss is bounded by $O(\epsilon)$, and the generalization error is $\Otil(\epsilon + 1/\sqrt{m})$. So neural networks trained with sufficiently many data points to a solution with small gradient in $\Gcal_W$ generalize well. 
This essentially means that although non-convex, the loss function over this set is well behaved, and thus learning is not difficult. 

Note that an immediate corollary is that any critical point in $\Gcal_W$ is global optimum. 

Furthermore, a randomly sampled set of weights $W$ are likely to fall into this set. This suggests a reason for the practical success of training with random initialization: after initialization, the parameters w.h.p. fall into the set, then would stay inside during training, and finally get to a point with small gradient, which by our analysis, has small error.

\paragraph{Remark 2.} 
An important feature about our result is that the set $\Gcal_W$ has a simple explicit form:
\begin{align} \label{eqn:gw}
	\Gcal_W = \cbr{W \in \Fcal_W: (L_2(W))^2 \le c_g \rbr{\sqrt{\frac{\log d}{nd} \log \frac{1}{\delta'}} + \frac{1}{n} \log \frac{1}{\delta'}} }
\end{align}
where $c_g > 0$ is a universal constant. Furthermore,  $(L_2(W))^2$ has a simple closed form (See Theorem~\ref{thm:stolarsky}). Therefore, it is possible to directly check if a solution $W$ is in $\Gcal_W$, or design regularization that make sure $W$ stays in the set $\Gcal_W$.

\paragraph{Remark 3.} 
The above theorem is a special case of the following more general result.

\begin{theorem}[Main, general] \label{thm:error} 
Let $\delta \in (0,1)$.
With probability $\ge 1 - m\exp\rbr{-m\gamma_m/8} - 2m^2\exp\rbr{-4\log^2 d} - \delta$, the following holds. For any $\xi, \eta > 0 $ and any $W$ with $L_2(W) = \Otil\rbr{n^{-\xi} d^{-\eta}}$ such that 
\begin{align}
  m^{\beta} \le \Otil\cbr{ d^{(1+\eta)/2} n^{\xi/2} m^{1/4}, d^{1/2} m^{1/2}, n^{\xi} d^{1/2 + \eta} }, \label{eqn:thm:error}
\end{align}
and any $v \in \cbr{-1, 1}^n$, we have
\begin{align*}
  & \frac{1}{2m} \sum_{\ell=1}^m (f(x_\ell; v, W) - y_\ell)^2  \le \frac{c m^\beta}{n} \nbr{\frac{\partial L}{\partial W}}^2, \\
	& \frac{1}{2} \EE (f(x; v, W) - y)^2 \le  \frac{c m^\beta}{n} \nbr{\frac{\partial L}{\partial W}}^2 + c'(C_W + Y)^2 \sqrt{\frac{1}{m}\log \frac{1}{\delta}}.
\end{align*}
\end{theorem}
Here $\tilde\Omega$ hides logarithmic terms $\log m \log \frac{1}{\delta} \log \frac{1}{\delta'}$.

The theorem shows that when the weights are diverse (\ie, with good discrepancy $L_2(W)$ so that $\xi$ and $\eta$ are sufficiently large), the training loss is proportional to $m^\beta /n$ times the squared norm of the gradient. This implies a local minimum leads to a global minimum and the neural network learns the target function. The generalization error has an additional term $\Otil\rbr{1/\sqrt{m}}$. 

To obtain Theorem~\ref{thm:main} from this more general result, we first note that Lemma~\ref{lem:l2discrepancy} proves most $W$ falls into $\Gcal_W$, so it is sufficient to choose $n, d$ large enough to guarantee the condition (\ref{eqn:thm:error}). Setting $n = \Otil(m^\beta)$ and $d = \Otil(m^\beta)$ satisfies the condition with $\xi=\eta=1/4$, and thus leads to Theorem~\ref{thm:main}. 
Clearly, there exist some other options. For example, $d = \tilde\Omega\rbr{m^{2\beta - 1}}$ and $n = \tilde\Omega\rbr{ m^{3 - 2\beta}}$, which matches the empirical observation that overspecified networks with large $n$ are easier for the optimization.

\section{Analysis roadmap}

Our key technical result is a lower bound on the smallest singular value of $D$ based on the spectrum of the activation kernel defined in \eq{eqn:relu_kernel} and the discrepancy of the weights defined in~\eq{eq:l2_dsp}. 
Once the lower bound is obtained, we can use (\ref{eqn:residual_bound}) to bound the training loss, and use Rademachar complexity to bound the generalization error. 

\begin{theorem}[Smallest singular value] \label{thm:singular_value}
With probability $\ge 1 - m\exp\rbr{-m\gamma_m/8} - 2m^2\exp\rbr{-4\log^2 d} - \delta$, the following holds. For any $\xi, \eta > 0$, any $W$ with $L_2(W) = \Otil\rbr{n^{-\xi} d^{-\eta}}$, and any $v \in \cbr{-1, 1}^n$,
\begin{align*}
s_m(D)^2  \ge \Omega(nm^{1-\beta}) - \Otil\Bigg(\frac{n^{1-\xi/2 } m^{3/4} }{d^{(1+\eta)/2}}\Bigg) 
  -\Otil\Bigg(\frac{n m^{1/2}}{d^{1/2}}\Bigg)  - \Otil\Bigg(\frac{n^{1-\xi}m}{d^{ 1/2 + \eta}} \Bigg).
\end{align*}
\end{theorem} 
Here the notation $\tilde\Omega$ hides logarithmic terms $\log d \log \frac{1}{\delta}$.

The theorem is stated in its general form. 
It bounds the smallest singular value in terms of the  $n, d, m$ and two parameters $\xi,\eta$ quantifying how large $L_2(W)$ is. 
It is instructive to consider the special case when $n = \tilde\Omega(m^{\beta})$, $d = \tilde\Omega(m^{\beta})$, and $\xi = \eta = 1/4$, which corresponds to Theorem~\ref{thm:main}).
In this case, with probability at least $1- cm^{-\log m} - \delta $,
$
  s_m(D)^2 \ge c m
$
for some constant $c>0$. It is clear that the singular value is large and bounded away from zero. 

It is interesting to compare the theorem to the results in~\cite{SouCar16}, which shows that $D$ is full rank with probability one under small perturbations. However, full-rankness alone is not sufficient since its smallest singular value could be extremely small leading to possibly huge training loss. Instead, we directly bound the smallest singular value and relate it to the activation and the diversity of the weights. 

\paragraph{Intuition.}
Here we describe the high level intuition for bounding the minimum singular value.
It is necessarily connected to the activation function and the diversity of the weights. For example, if $\sigma'(t)$ is very small for all $t$, then the smallest singular value is expected to be very small. For the weights, if $d < m$ (the interesting case) and all $w_k$'s are the same, then $D$ cannot have rank $m$. If $w_k$'s are very similar to each other, then one would expect the smallest singular value to be very small or even zero. Therefore, some notion of diversity of the weights is needed. 

The analysis begins by considering the matrix $G_n = D^\top D/n$. It suffices to bound $\lambda_m(G_n)$, the $m$-th (and the smallest) eigenvalue of $G_n$. 
To do so, we introduce a matrix $G$ whose entries $G(i,j) = \EE_w [G_n(i,j)]$ where the expectation $\EE_w$ is taken assuming $w_k$'s are uniformly random on the unit sphere. The intuition is that when $w$ is uniformly distributed, $\sigma'(w^\top x)$ is most independent of the actual value of the $x$, and the matrix $D$ should have the highest chance of having large smallest singular value. We will introduce $G$ as an intermediate quantity and subsequently bound the spectral difference between $G_n$ and $G$.
Roughly speaking, we break
the proof into two steps
\begin{align*}
  \lambda_m(G_n) \geq \underbrace{\lambda_m(G)}_{\text{I. ideal spectrum}} - \underbrace{\nbr{G - G_n}}_{\text{II. discrepancy}}
\end{align*}
where $\nbr{G - G_n}$ is the spectral norm of the difference. 

For the first term in the lower bound, we observe that $G$ has a particular nice form: $G(i,j) = g(x_i, x_j)$, the kernel defined in (\ref{eqn:relu_kernel}). This allows us to apply the eigendecomposition of the kernel and positive definite matrix concentration inequality to bound $\lambda_m(G)$, which turns out to be around $m \gamma_m/2$. 

For the second term, when $w_k$'s are indeed from the uniform distribution over the sphere, this can be bounded by concentration bounds. It turns out that when $w_k$'s are not too far away from that, it is still possible to do so. Therefore, we use the geometric discrepancy to measure the diversity of the weights, and show that when they are sufficiently diverse, $\nbr{ G - G_n }$ is small. 
In particular, the entries in $G - G_n$ can be viewed as the kernel of some U-statistics, hence concentration bounds can be applied. The expected U-statistics turns out to be the $(L_2(W))^2$, which has a closed form and can be shown to be small. 

\paragraph{Outline.} 
Theorem~\ref{thm:singular_value} is proved in Section~\ref{sec:singular}, $L_2(W)$ and $\Gcal_W$ are characterized in Section~\ref{sec:l2disc}, and the proof sketch of Theorem~\ref{thm:main} and Theorem~\ref{thm:error} is provided in Section~\ref{sec:error}. 
We describe the proof sketch for the lemmas and provide the remaining proofs in the appendix.

\vspace{-3mm}
\section{Bounding the smallest singular value} \label{sec:singular}
\vspace{-2mm}

Theorem~\ref{thm:singular_value} can be obtained from the following technical lemma.

\begin{lemma} \label{lem:smallest}
With probability $\ge 1 - m\exp\rbr{-m\gamma_m/8} - 2m^2\exp\rbr{-4\log^2 d} - \delta$, we have
\begin{align}
s_m(D)^2 &\geq nm\gamma_m/2  - c n \rho(W),
\end{align}
where
\begin{align}
\rho(W) &= \frac{\log d}{\sqrt{d}}  \sqrt{ L_\infty(W) L_2(W) } m\rbr{\frac{4}{m}\log\frac{1}{\delta}}^{1/4} +  \frac{\log d}{\sqrt{d}} m L_\infty(W)\sqrt{\frac{4}{3m}\log \frac{1}{\delta}} 
+ \frac{\log d}{\sqrt{d}} m L_2(W) + L_\infty(W).
\end{align}
\end{lemma}

\begin{proof}[Proof of Theorem~\ref{thm:singular_value}]
First, $\abr{\text{dsp}(W, S)} \le 2$ for any set $W$ and slice $S$, so by definition $\abr{L_\infty(W)} \le 2$. 
Next, By the assumption in the theorem, $L_2(W) = \Otil(n^{-\xi} d^{-\eta})$. Plugging these into Lemma~\ref{lem:smallest} completes the proof.
\end{proof}
Lemma~\ref{lem:smallest} is meaningful only when $cn\rho(W)$ is small compared to $nm\gamma_m/2$. This requires $L_2(W)$ to be sufficiently small. 
In the following we will first provide the proof sketch of Lemma~\ref{lem:smallest}, and then bound that $L_2(W)$ in Section~\ref{sec:l2disc}.




\vspace{-3mm}
\subsection{Proof of Lemma~\ref{lem:smallest} } \label{sec:proof_smallest}
\vspace{-2mm}

To prove Lemma~\ref{lem:smallest}, it is sufficient to bound the smallest eigenvalue of $G_n = D^\top D/n$. 
Note that $v_k \in \cbr{-1, 1}$, so $v_k^2 = 1$, and thus the $(i,j)$-th entry of $G_n$ is
\begin{align}
  G_n(i,j) = \frac{1}{n} \sum_{k=1}^n \sigma^{\prime}(w_k^\top x_i)  \sigma^{\prime}(w_k^\top x_j) \inner{x_i}{x_j} .
\end{align}

For ReLU, $\sigma'(w^\top x)$ does not depend on the norm of $w$ so without loss of generality, we assume $\nbr{w}=1$.
Consider a related matrix $G$ whose $(i,j)$-th entry is defined as
\begin{align} 
  G(i,j) = \EE_{w}\sbr{\sigma^{\prime}(w^\top x_i)  \sigma^{\prime}(w^\top x_j) \inner{x_i}{x_j}} .
\end{align} 
where $w$ is a random variable distributed uniformly over the unit sphere. 
Since $\sigma'(w^\top x) = \II\sbr{w^\top x \ge 0}$, we have a closed form expression for $G(i,j)$:
\begin{align}
G(i,j) &= \EE_{w}\sbr{\II(w^\top x_i \ge 0) \II(w^\top x_j \ge 0)} \inner{ x_i}{ x_j} \nonumber\\
&= \rbr{\frac{1}{2} - \frac{\arccos\inner{x_i}{x_j}}{2\pi}}  \inner{ x_i}{ x_j}.
\end{align}
Note that $G(i,j) = g(x_i, x_j)$ where $g$ is the kernel defined in (\ref{eqn:relu_kernel}). This allows us to reason about the eigenspectrum of $G$, denoted as $\lambda_1(G) \ge \ldots \ge \lambda_m(G)$.

Therefore, our strategy is to first bound $\eigg{m}$ in Lemma~\ref{lem:lowerG} and then bound $|\lambda_m(G) - \lambda_m(G_n)|$ in Lemma~\ref{lem:g_difference}.  Combining the two immediately leads to Lemma~\ref{lem:smallest}.

{\bf First, consider $\lambda_m(G)$.} 
We consider a truncated version of spherical harmonic decomposition:
\[
 g^{[m]}(x_i, x_j) = \sum_{u=1}^m \gamma_u \phi_u(x_i)\phi_u(x_j) 
\]
and the corresponding matrix $G^{[m]}$. On one hand, it is clear that $\lambda_m(G) \ge \lambda_m(G^{[m]})$.
On the other hand, $G^{[m]} = AA^\top$ where $A$ is a random matrix whose rows are
$$
   A^i := [\sqrt{\gamma_1} \phi_1(x_i), \ldots,  \sqrt{\gamma_m} \phi_m(x_i)].
$$ 
Next, we bound $\lambda_m(G^{[m]})$ by matrix Chernoff bound~\cite{Tropp12}, and it is better than previous work~\cite{Braun06}.
This leads to the following lemma. 

\begin{lemma}\label{lem:lowerG}
With probability at least $1-m \exp\rbr{ -m \gamma_m / 8 }$,
$$ 
   \lambda_m(G) \ge m\gamma_m/2.
$$
\end{lemma}

{\bf Next, bound $|\lambda_m(G) - \lambda_m(G_n)|$.} By Weyl's theorem, this is bounded by $\nbr{G - G_n}$. 
To simplify the notation, denote
\begin{align*}
E_{i,j}  =  & ~ \EE_w[\sigma^\prime(w^\top x_i) \sigma^\prime(w^\top x_j)] ~ - \frac{1}{n}\sum_{k=1}^n \sigma^\prime(w_k^\top x_i) \sigma^\prime(w_k^\top x_j).
\end{align*}
Then $G(i,j) - G_n(i,j) = \inner{x_i}{x_j} E_{ij}$, and thus 
\begin{align}
& \quad \nbr{G - G_n} \nonumber \\
& = \sup_{\nbr{\alpha}=1} \sum_{i,j} \alpha_i \alpha_j  \inner{x_i}{x_j} E_{ij}  \nonumber\\
& \le \sup_{\nbr{\alpha} = 1}  \sqrt{\sum_{i\neq j} \alpha_i^2 \alpha_j^2} \sqrt{\sum_{i\neq j} \inner{x_i}{x_j}^2 E_{ij}^2} + \max_i \abr{E_{ii}} \nonumber\\
&\le  \frac{c\log d}{ \sqrt{d}} \sqrt{\sum_{i\neq j}  E_{ij}^2} + \max_i \abr{E_{ii}}  
  \label{eqn:bound_diff}
\end{align}
where the last inequality holds with high probability since $x_i$'s are uniform over the unit sphere and thus we can apply sub-gaussian concentration bounds.  

Note that $\sum_{i\neq j}  E_{ij}^2 / (m(m-1))$ is a U-statistics where the summands are dependent and typical concentration inequality for~\iid~entries does not apply. Instead we use a Bernstein inequality for U-statistics~\cite{PeeAntRal10} to show that 
with probability at least $1- \delta$, it is bounded by
\begin{align} \label{eqn:eij}
\EE_{\cbr{x_1,x_2}} E_{12}^2 + \sqrt{\frac{4\Sigma^2}{m}\log\frac{1}{\delta}} 
  + \frac{4B^2}{3m}\log \frac{1}{\delta} 
\end{align}
where $B = \max_i |E_{ii}|$ and $\Sigma^2 = \EE\sbr{E_{12}^4} - (\EE\sbr{E_{12}^2})^2$.

The key observation is that the quantities in the above lemma are related to discrepancy:
\begin{align}
\max_{i,j} E_{ij}  &\le  L_\infty(W), \label{eqn:eij1}\\
\EE_{x_1,x_2}\sbr{E_{12}^2} &= (L_2(W))^2, \label{eqn:eij2}\\
\Sigma^2 & \le (L_2(W) L_\infty(W))^2. \label{eqn:eij3}
\end{align}
This is because 
$
  \sigma^\prime(w^\top x_i) \sigma^\prime(w^\top x_j) = \II[w \in S_{x_i x_j}]
$
and thus 
\begin{align*}
E_{i,j} & = \EE_w \II[w \in S_{x_i x_j}]  - \frac{1}{n}\sum_{k=1}^n \II[w_k \in S_{x_i x_j}] \\
& = A(S_{x_i x_j}) - \frac{1}{n} \abr{W \cap S_{x_i x_j}}\\
& = -\text{dsp}(W, S_{x_i x_j}).
\end{align*}

Plugging (\ref{eqn:eij1})-(\ref{eqn:eij3}) into (\ref{eqn:eij}) and (\ref{eqn:bound_diff}), we have
\begin{lemma}\label{lem:g_difference}
The following inequality holds with probability at least $1 - 2m^2\exp\rbr{-\log^2 d} -\delta$, 
\begin{align}
\nbr{G_n - G} &\le c \rho(W)
\end{align}
where $\rho(W)$ is as defined in Lemma~\ref{lem:smallest}. 
\end{lemma}

\vspace{-3mm}
\subsection{Characterizing the discrepancy} \label{sec:l2disc}
\vspace{-2mm}

In this subsection, we present a bound for $L_2(W)$ and show that the $\Gcal_W$ defined in the following covers most $W$'s. 
Recall that
\begin{align*}
	\Gcal_W = \cbr{W: (L_2(W))^2 \le c_g \rbr{\sqrt{\frac{\log d}{nd} \log \frac{1}{\delta'}} + \frac{1}{n} \log \frac{1}{\delta'}} }
\end{align*}
for $0 < \delta' < 1$ and a proper constant $c_g > 0$. The constant $c_g$ is the constant in Lemma~\ref{lem:l2discrepancy}. $\delta'$ will be clear from the context where $\Gcal_W$ is used.

First we provide a closed form for $L_2$ discrepancy of slices defined in~(\ref{eqn:slice}). The proof is provided in the appendix.
\begin{restatable}{theorem}{stolarsky} \label{thm:stolarsky}
Suppose $W = \cbr{w_i}_{i=1}^n \subseteq \mathbb{S}^{d-1}$.
\begin{align*}
        \rbr{L_2(W)}^2 & = \frac{1}{n^2} \sum_{i,j = 1}^n k(w_i,w_j)^2  - \EE_{u,v} \sbr{ k(u,v)^2 }
\end{align*}
where $\EE_{u,v}$ is over $u$ and $v$ uniformly distributed on $\mathbb{S}^{d-1}$ and the kernel $k(\cdot,\cdot)$ is
\[
  k(u,v) = \frac{\pi - \arccos \inner{u}{v} }{2 \pi}.
\]
\end{restatable}
The closed form is simple and intuitive. The kernel $k(w_i, w_j)$ measures how similar two units are. The discrepancy is the difference between the average pairwise similarity and the expected one over uniform distribution. 

Given the theorem, we now show that $\rbr{L_2(W)}^2$ can be small. We use the probabilistic method, \ie, show that if $w_k$'s are sampled from $\mathbb{S}^{d-1}$ uniformly at random, then with high probability $W$ falls into $\Gcal_W$. The key observation is that with random $W$, $(L_2(W))^2$ is the difference between a U-statistics and its expectation, which can be bounded by concentration inequalities. Formally,

\begin{restatable}{lemma}{discrepancybound} \label{lem:l2discrepancy}
There exists a constant $c_g$, such that for any $0 < \delta' < 1$, with probability at least $1-\delta'$ over $W = \cbr{w_i}_{i=1}^n$ that are sampled from the unit sphere uniformly at random, 
\begin{align*}
  \rbr{L_2(W)}^2 \le c_g\rbr{\sqrt{\frac{\log d}{nd} \log \frac{1}{\delta'}} + \frac{1}{n} \log \frac{1}{\delta'}}.
\end{align*}
\end{restatable}

Alternatively, the theorem means that $\Gcal_W$ defined in (\ref{eqn:gw}) covers most $W$'s. This is because $L_2(W)$ is independent of the length of $w_k$'s, it is sufficient to show that $L_2(W)$ is small for $W \in \Gcal_W \cap \mathbb{S}^{d-1}$.

%
%

\vspace{-3mm}
\section{Final bound on generalization error} \label{sec:error}
\vspace{-2mm}

Here we provide the proof sketch of Theorem~\ref{thm:error} and Theorem~\ref{thm:main}.
More details of the proof are in Appendix~\ref{sec:rademacher}.

First, we prove Theorem~\ref{thm:error}. Suppose a solution $W$ satisfies the assumption and has small gradient $\nbr{\partial L /\partial W}$. 
Using (\ref{eqn:residual_bound}), we have
$
\nbr{r} \le {\nbr{\partial L/\partial W}}/{s_m(D)}.
$
By Theorem~\ref{thm:singular_value} and the assumption in Theorem~\ref{thm:error}, with high probability $s_m^2(D) = \Omega(nm^{1-\beta})$. This implies the training loss is 
\begin{align*}
\frac{1}{m}\sum_{l} \rbr{f(x_l) - y_l}^2  = m\nbr{r}^2 
\le \frac{c m^\beta}{n} \nbr{\frac{\partial L}{\partial W}}^2.
\end{align*}

The generalization error can then be derived using McDiamid's inequality and Rademacher complexity.
First, we need an upper bound on the difference of the loss for two data points for the McDiamid's inequality. Since $\nbr{x}_2 \le 1$ and $\sum_k\nbr{w_k}_2 \le C_W$, we have $\abr{f} \le C_W$. Thus
\begin{align*}
& \quad \abr{ \ell(y, f(x)) - \ell(y', f(x')) } \\
& \le \frac{1}{2}\max\cbr{ (y - f(x))^2 , (y'- f(x'))^2} \le Y^2 + C_W^2 
\end{align*}
where in the last inequality we use the fact that the true function $|y| \le Y$.
Next, we use the composition rules to compute the Rademacher complexity. Since the complexity of linear functions $\cbr{w^\top x : \nbr{w}_2\le b_W, \nbr{x}_2\le1}$ is $b_W / \sqrt{m}$ and $\sigma(\cdot)$ is 1-Lipschitz, and $\sum_k \nbr{w_k}_2 \le C_W$, the complexity 
$
\Rcal_m(\Fcal) \le C_W/\sqrt{m}.
$ 
Composing it with the loss function, and applying the bound in~\cite{BarMen02}, we get the final generalization bound.

To obtain Theorem~\ref{thm:main} from the more general Theorem~\ref{thm:error}, we first note that Lemma~\ref{lem:l2discrepancy} proves most $W$ falls into $\Gcal_W$, and setting $n = \Otil(m^\beta)$ and $d = \Otil(m^\beta)$ makes sure that any $W \in \Gcal_W$ has $L_2(W) = \Otil(n^{-1/4}d^{-1/4})$ satisfying the condition (\ref{eqn:thm:error}).  

\section{Discussions}
In this section, we discuss and remark on further considerations and possible extensions of our current analysis. 

\subsection{Other loss functions}
Currently, our analysis is tied to the least squares loss $\ell(y, f(x)) = \frac{1}{2}\rbr{y - f(x)}^2$. It is fairly straightforward to generalize it to any strongly convex loss function, such as logistic loss. Note that the final objective function is \emph{not} convex due to the non-convexity in neural networks, but most loss functions used in practice are strongly convex w.r.t. $f(x)$.
Under the new setting, the residual is then
$$r = \frac{1}{m}\rbr{\ell'(y_1, f(x_1)), \cdots, \ell'(y_m, f(x_m))}^\top.$$
According to our analysis, the norm of the residual $\nbr{r}$ will be bounded. This in turn implies each individual $\ell'(y_l, f(x_l))$ will be small. Since the loss function $\ell(y, f(x))$ is strongly convex, the loss itself will be small.

\subsection{Other activation functions}
We can consider a family of activation functions of the form $\sigma(u) = \max\cbr{u, 0}^t$, \ie, rectified polynomials~\cite{ChoSaul09,KroHop16}. This requires two modifications to the analysis.

One is the corresponding kernel $k(x,y) = \EE_w\sbr{\sigma'(w^\top x)\sigma'(w^\top y)}$ and $g(x,y) = k(x,y)\inner{x}{y}$. When the input distribution is uniform, we can also compute the kernels in closed form as shown in~\cite{ChoSaul09}:
\begin{align}\label{eqn:rectified_poly}
k_t(x, y) = \frac{J_{t-1}(\theta)}{2\pi}
\end{align}
where
\begin{align}
J_t(\theta) = (-1)^t (\sin \theta)^{2t+1}\rbr{\frac{1}{\sin\theta}\frac{\partial}{\partial \theta}}^t \rbr{\frac{\pi-\theta}{\sin\theta}},
\end{align}
and $\theta = \arccos\inner{x}{y}$. Note that the subscript is $t-1$ in (\ref{eqn:rectified_poly}) because we are computing on the derivative $\sigma'(u)$.

Examples for the first few $t$ are listed as follows.
\begin{align}
J_0(\theta) &= \pi - \theta \\
J_1(\theta) &= \sin\theta + (\pi -\theta) \cos\theta \\
J_2(\theta) &= 3\sin\theta\cos\theta + (\pi -\theta)(1+ 2\cos^2\theta)
\end{align}
Larger $t$ corresponds to more nonlinear activation functions and leads to slower decaying spectrum since there are more high frequency components.

We also need to change the definition of the discrepancy to accommodate the new kernels. Let
\begin{align}
\rbr{L_2(W)}^2 &= \EE_{x_i, x_j}\sbr{\EE_w[\sigma^\prime(w^\top x_i) \sigma^\prime(w^\top x_j)] ~ - \frac{1}{n}\sum_{k=1}^n \sigma^\prime(w_k^\top x_i) \sigma^\prime(w_k^\top x_j) }^2    \nonumber\\
&= \frac{1}{n^2} \sum_{i,j = 1}^n k(w_i,w_j)^2  - \EE_{u,v} \sbr{ k(u,v)^2 }.
\end{align}
Therefore, the discrepancy is affected by how the kernels change due to change in activation functions.

The other modification is on the Rademacher complexity. Since the derivative $\sigma'(u) = t \max\cbr{u,0}^{t-1}$,
there is an additional factor of $t$ in front of the complexity. That is, larger $t$ leads to higher Rademacher complexity.

In summary, the best parameter $t$ depends on the balance between the two conflicting effects. On one hand, larger $t$ corresponds to slower decaying spectrum and makes the minimum singular value more likely to be larger. On the other hand, smaller $t$ leads to better generalization since the Rademacher complexity is smaller.

\subsection{(Sub)gradient of the activation function}
Throughout this paper, we have used one particular subgradient for the ReLU activation function: $\II\sbr{u > 0}$. At the point $u=0$, there are many other valid subgradients as long as its value is between 0 and 1. However, our analysis is not restricted to this particular subgradient. First of all, all the subgradients only differ at one point $u=0$, which is of probability zero. Second, our analysis is probabilistic in nature. The first term in Lemma~\ref{lem:smallest} is the expectation over $W$, which is insensitive to the probability zero event $u=0$. The second term in Lemma~\ref{lem:smallest} is related to $L_2(W)$, which is again expectation over all possible data, thus insensitive to the difference.

In summary, though for some $W \in \Gcal_W$ the loss is not differentiable, one can define $\partial L /\partial W$ by using subgradients of ReLU $\sigma$ as follows:
\begin{align} \label{def:relu_subg}
  \sigma'(x) = 
	\begin{cases}
	  0,  & x < 0 \\
		c,  & x = 0 \\
		1,  & x > 0
	\end{cases} 
\end{align}
for any $c \in [0,1]$. 
Then under the conditions in our theorems, with high probability, for any $W \in \Gcal_W$ and any definition of $\sigma'$ in~(\ref{def:relu_subg}), the guarantees hold. 

Other activation functions such as rectified polynomials are differentiable and thus they do not have such issue.

\subsection{Other input distribution}
When the input distribution is not uniform, the spectrum of the kernel function defined in (\ref{eqn:gxy}) will be different because the spherical harmonic bases are defined with respect to the input distribution.
To ensure the spectrum decays slowly, we need to find a corresponding distribution of $W$ that ``matches'' the input distribution.

We provide some intuitions in finding such ``matching'' distribution. Suppose the input distribution is uniform on the set $E \in \mathbb{S}^{d-1}$, if a hyperplane whose normal is $w$ does not ``cut through'' the set, then for all data points, they have the same sign $\II[w^\top x > 0]$. This will likely lead to rank deficiency in the extended feature matrix.

Therefore,  we prefer $W$ to split the data points as much as possible. One such distribution of $W$ is uniform on the set $F_E = \cbr{w \in \mathbb{S}^{d-1} : \textnormal{there exists}~~u \in E, ~ \inner{u}{w} = 0}$. For example, if $E$ is the intersection of the positive orthant and the unit sphere, $E = \cbr{u \in \mathbb{S}^{d-1} : u_i \ge 0,~\textnormal{for all}~~i \in [d]}$, then the corresponding set $F_E$ is the whole sphere excluding $E$ and $-E$. 

We have verified the phenomenon empirically. We have generated 3000 input data points uniform on the positive orthant $E$. We then compute the $3000 \times 3000$ Gram matrix, where the $(i,j)$-th entry is $\EE_w\sbr{\sigma'(w^\top x_i)\sigma'(w^\top x_i)\inner{x_i}{x_j}}$. The expectation is approximated by sampling 100,000 independent $w$'s and then averaging. We compare two distributions of $W$: 1) uniform on the whole unit sphere; 2) uniform on $F_E$. 

\begin{table}
\centering
\caption{Comparison of minimum eigenvalues with uniform and ``matching'' distributions. Note that the ``matching'' distribution corresponds to larger minimum eigenvalue for different dimensions. However, the difference becomes negligible when the dimension increases.}\label{tbl:dim_distribution}
\begin{tabular}{ccccc}
\toprule
$d$ & 4 & 5 & 6 &7 \\
\midrule
uniform & $3.96 \times 10^{-4}$ & $0.0015$ & $0.0032$ & $0.0072$ \\
\midrule
matching & $\bf{5.43 \times 10^{-4}}$ & $\bf{0.0017}$ & $0.0032$ & $0.0072$ \\
\bottomrule
\end{tabular}
\end{table}

In Table~\ref{tbl:dim_distribution}, we compare the minimum eigenvalues with the two distributions. The uniform distribution on $F_E$ always leads to larger or the same minimum eigenvalues. However, as dimension increases, the difference becomes negligible.
Note that the difference between the uniform distribution on the whole sphere and uniform on $F_E$ becomes exponentially small when the dimension $d$ increases, because the proportion of $E$ and $-E$ shrinks exponentially. This suggests that in high dimensions, uniform on the whole unit sphere is a reasonable distribution for $W$.

For a general input distribution $P(x)$, we can decompose it into small sets $dx$ and on every set, the distribution is uniform with measure $P(x)dx$. Then every small sets corresponds to a distribution of $W$. The final distribution of $W$ is the superposition of all such distributions, weighted by $P(x)dx$.

\vspace{-3mm}
\section{Numerical evaluation}
\vspace{-2mm}

In this section, we further investigate numerically the effects of gradient descent on the discrepancy and the effects of regularizing the weights using discrepancy measure.

\vspace{-3mm}
\subsection{Discrepancy and gradient descent}
\vspace{-2mm}

One limitation of the analysis is that we have not analyzed how to obtain a solution $W \in \Gcal_W$ with small gradient. The theoretical analysis of gradient descent is left for future work. Meanwhile we provide some numerical results supporting our claims. 

Although the set $\Gcal_W$ contains most $W$'s, it is still unclear whether the solutions given by gradient descent lie in the set. We design experiments to investigate this issue. The ground truth input data are of dimension $d=50$ and true function consists of $n=50$ units. We use networks of different $n$ to learn the true function and perform SGD with batch size 100 and learning rate 0.1 for 5000 iterations. Figure~\ref{fig:discrepancy} shows how $\rbr{L_2(W)}^2$ changes as a function of $n$. It is slightly worse than $O(n^{-1})$ but scales better than $O(n^{-1/2})$, suggesting (stochastic) gradient descent outputs solutions with reasonable discrepancy.

\begin{figure}
\centering
\begin{minipage}[b]{0.48\textwidth}
\centering
\includegraphics[width=0.8\textwidth]{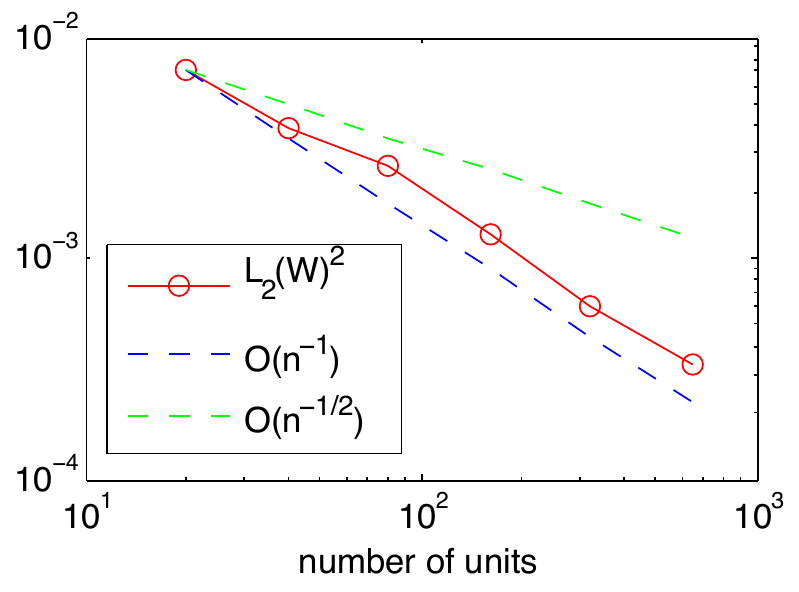}
\caption{Discrepancy of $W$ obtained after gradient descent. We perform gradient descent and compute discrepancy for the returned solution. The red curve corresponds to such solutions with different $n$. It scales similarly to the bound for uniform $W$ as in Lemma~\ref{lem:l2discrepancy}.}
\label{fig:discrepancy}
\end{minipage}
\hfill
\begin{minipage}[b]{0.48\textwidth}
\centering
\includegraphics[width=0.8\textwidth]{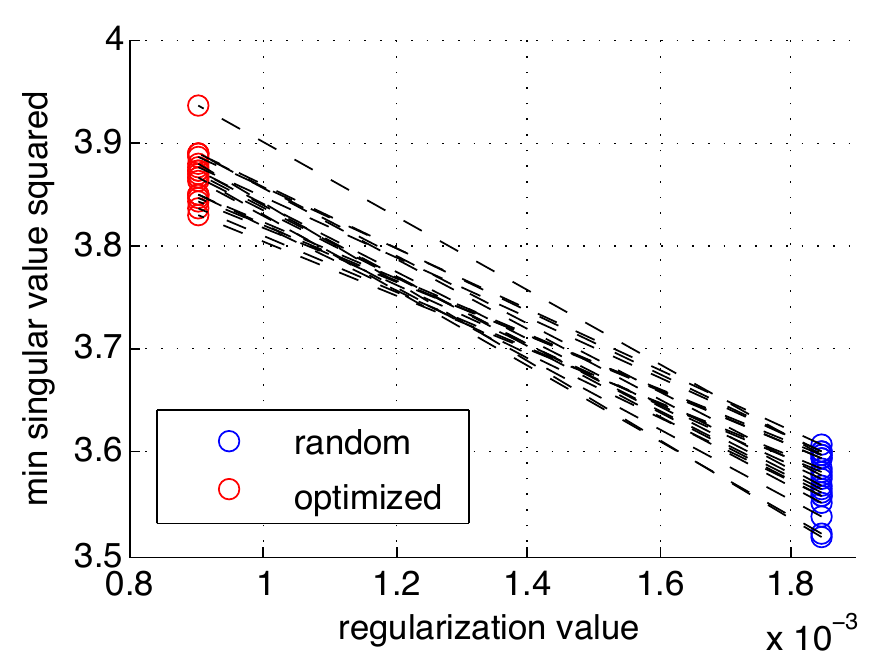}
\caption{Effect of regularization. The blue dots represent random weights and the red dots linked with dashed black lines represent weights optimized by minimizing $R(w)$. Smaller regularization values correspond to larger minimal singular values.}\label{fig:regularization}
\end{minipage}
\end{figure}


\vspace{-3mm}
\subsection{Regularization}
\vspace{-2mm}

To reinforce solutions with small discrepancy, we propose a novel regularization term to minimize $L_2$ discrepancy:
\begin{align}
R(W) &= \frac{1}{n(n-1)} \sum_{i \neq j}^n k(w_i,w_j)^2.
\end{align}
It is essentially $L_2$ discrepancy without the constants.

To verify the effectiveness of the regularization term, we explore the relationship between the regularization and the minimum singular value. We first generate 20 random $W$'s, all with $n=100$ and $d=100$, and compute their discrepancy and singular values using $m=3000$. Then we optimize $R(W)$ and compare the quantities after optimization. The result is presented in Figure~\ref{fig:regularization}. We can see smaller regularization value corresponds to larger singular value.

We also conducts experiments to compare training and test errors with and without regularization. The ground truth data are of $d=100$ and $n=100$. We learn the true function by SGD with learning rate 0.1, momentum 0.9 and a total of 300,000 iterations. The regularization coefficients are chosen from $\cbr{1,0.1, 0.01, 0.001}$ and the best results are reported. We use neural networks of size $n \in \cbr{100, 150, 200, 300}$ and for each $n$ we repeat five times with different random seeds. The result is summarized in Table~\ref{tbl:regularization}. 
Regularization leads to lower training and test errors for most settings. Even in the case where the un-regularized one performs better, the errors are all small enough (within the same range as standard deviation), suggesting the noise begins to dominate.

\begin{table}
\centering
{
\caption{Comparison of performance with/without regularization (all numbers are of unit $10^{-5}$).
The true function is generated with $d=100$ and $n=100$. To learn the function, we use networks with different $n$. }\label{tbl:regularization}
\vspace{-3mm}
\begin{tabular}{ccccc}
\toprule
 & \multicolumn{2}{c}{$n=100$} & \multicolumn{2}{c}{$n=150$}    \\
\cmidrule{2-5}
& train & test & train & test \\
\midrule
no-reg & 15.42(5.86) 	& 14.80(5.36)			& 1.79(0.45) & 1.86(0.50)   \\
\midrule
reg	 & \textbf{11.32(1.77)} & \textbf{10.63(1.58)} & \textbf{1.07(0.84)} & \textbf{1.13(0.99)}   \\
\bottomrule
\toprule
 & \multicolumn{2}{c}{$n=200$} & \multicolumn{2}{c}{$n=300$}    \\
\cmidrule{2-5}
& train & test & train & test \\
\midrule
no-reg & \textbf{0.38(0.27)} & \textbf{0.44(0.35)} & 0.39(0.39) & 0.44(0.40)   \\
\midrule
reg		 & 0.50(0.51) 	& 0.58(0.59)		& \textbf{0.10(0.05)} & \textbf{0.12(0.07)}  \\
\bottomrule
\end{tabular}}
\vspace{-3mm}
\end{table}

We also compare the regularization effects on the MNIST dataset. The dataset contains 60,000 training and 10,000 test handwritten digits. To demonstrate the regularization effect, we train one hidden layer fully connected neural networks with $k=200, 400, 600, 800$ units. The results are summarized in Table~\ref{tbl:regularization_mnist}. Note that state-of-the-arts performance on MNIST are mostly obtained by convolutional neural networks. This experiment is not intended to achieve the state-of-the-arts but it tries to showcase the advantage of regularization on a real-world dataset. 

From Table~\ref{tbl:regularization_mnist}, we see regularization consistently leads to slightly better test error for all cases.

\begin{table}
\centering
{
\caption{Performance comparison with/without regularization on MNIST dataset. Errors are all in $\%$.}\label{tbl:regularization_mnist}
\vspace{-3mm}
\begin{tabular}{ccccc}
\toprule
 & \multicolumn{2}{c}{$n=200$} & \multicolumn{2}{c}{$n=400$}    \\
\cmidrule{2-5}
& train & test & train & test \\
\midrule
no-reg & 0.94 	& 3.39			& \textbf{0.32} & 3.08   \\
\midrule
reg	 & \textbf{0.56} & \textbf{3.22} & {0.33} & \textbf{2.90}   \\
\bottomrule
\toprule
 & \multicolumn{2}{c}{$n=600$} & \multicolumn{2}{c}{$n=800$}    \\
\cmidrule{2-5}
& train & test & train & test \\
\midrule
no-reg & 0.00065 & 2.67 & 0.11 & 2.90   \\
\midrule
reg		 & \textbf{0.00057}  & \textbf{2.62}	& \textbf{0.0003} & \textbf{2.45}  \\
\bottomrule
\end{tabular}}
\vspace{-3mm}
\end{table}

\vspace{-3mm}
\section{Conclusion}
\vspace{-2mm}

We have analyzed one-hidden-layer neural networks and identified novel conditions when local optima become global optima despite the non-convexity of the loss function.
The key factors are the spectrum of the kernel associated with the activation function and the diversity of the units measured by discrepancy.

Although we focus on a least-square loss function and uniform input distribution, the analysis technique can be readily extended to other loss function and input distributions. At the moment, our analysis is still limited in the sense that it is independent of the actual algorithm. In the future work, we will explore the interplay between the discrepancy and gradient descent. In addition, we will further investigate the issue of designing an algorithm that guarantees good discrepancy thus small errors, possibly in a way similar to~\cite{GeLeeMa16} in low-rank recovery problems.

\section*{Acknowledge}

We thank Santosh Vempala, Lorenzo Rosasco and Jason Lee for valuable discussions. The research is supported in part by 
NSF/NIH BIGDATA 1R01GM108341, ONRN00014-15-1-2340, NSF IIS-1639792, NSF IIS-1218749, NSF CAREER IIS-1350983, Intel and NVIDIA, and by NSF grants CCF-0832797, CCF-1117309, CCF-1302518, DMS-1317308, Simons Investigator Award, and Simons Collaboration Grant.

\bibliographystyle{apalike}
  \bibliography{../../bibfile/bibfile}

\newpage
\appendix

\onecolumn

\section{Spherical harmonic decomposition and kernel spectrum}\label{app:spherical_spectrum}
Any function defined on the unit sphere has a spherical harmonic decomposition
\begin{align}
g(x, y) = \sum_u \gamma_u \phi_u(x) \phi_u(y),
\end{align}
where $\phi_u(x) : \mathbb{S}^{d-1} \mapsto \RR$ is a spherical harmonic basis. Note that $u=(t,j)$ is a multi-index: the first denotes the order of the basis and the latter denotes the index of bases with the same order.

For each order $t$, there are $N(d, t)=\frac{2t+d-2}{t}\mychoose{t+d-3}{t-1}$ bases with the same coefficient. As a result, the spectrum $\gamma_u$ sorted by magnitude has the step like shape where each step is of length $N(d,t)$.

To compute the coefficients, we use the Legendre harmonics~\cite{Mueller12} with the following property
\begin{align}
P_{t,d}(\inner{x}{y}) = \frac{1}{N(d,t)}\sum_{j=1}^{N(d,t)}\phi_{t,j}(x) \phi_{t,j}(y).
\end{align}

The spherical harmonics also form an orthonormal basis on the unit sphere: 
\begin{align}
\EE\sbr{\phi_{l,i}(x)\phi_{m,j}(x)} =\frac{1}{\abr{\mathbb{S}^{d-1}}} \int_{\mathbb{S}^{d-1}} \phi_{l,i}(x)\phi_{m,j}(x) dx=\delta_{lm}\delta_{ij},
\end{align}
where $\abr{\mathbb{S}^{d-1}}$ denotes the surface area of the unit sphere.

Combining these properties, we can calculate the spectrum using
\begin{align}
\gamma_{(t,j)} = \frac{\abr{\mathbb{S}^{d-2}}}{\abr{\mathbb{S}^{d-1}}} \int_{-1}^{1} g(\xi) P_{t,d}(\xi) (1-\xi^2)^{(d-3)/2} d\xi , ~~\textnormal{for all}~j \in \sbr{N(d,t)}.
\end{align}

\section{Bounding $\eigg{m}$ using matrix concentration bound: Proof of Lemma~\ref{lem:lowerG}} \label{app:lowerG}

Recall that
\begin{align}
  g(x, y) & = \sum_{u=1}^{\infty} \gamma_u \phi_u(x)\phi_u(y).
\end{align}

For an integer $r>0$, define the truncated version of $g$ and the corresponding residue as
\begin{align} 
  g^{[r]}(x, y) & = \sum_{u=1}^r \gamma_u \phi_u(x)\phi_u(y) \\
	e^r(x,y) & = g(x,y) - g^{[r]}(x,y) \nonumber\\
\end{align}
and define the matrices
\begin{align}
  \sbr{G^{[r]}}_{i,j} & = g^{[r]}(x_i, x_j) \nonumber\\
	E^r & = G - G^{[r]}.
\end{align}

\begin{lemma} \label{lem:chernoff_gm}
Let $c_g = \max_x g(x, x)$ then with probability at least $1-m \exp\rbr{ -\frac{m \gamma_m }{8\constg}}$,
$$ 
   \lambda_m( G^{[m]}) \ge  m\gamma_m/2.
$$
\end{lemma}

\begin{proof}
Define a matrix $A$ whose rows are 
$$
   A^i := [\sqrt{\gamma_1} \phi_1(x_i), \ldots,  \sqrt{\gamma_m} \phi_m(x_i)]
$$
for $1\le i \le m$. Define matrices
$$
  X_i = (A^i)^\top A^i.
$$
Denote $Y = \sum_{i=1}^m X_i $. Then $\lambda_m(\EE Y) = m\gamma_m$ using the fact that $\EE\sbr{ \phi_i(x) \phi_j(x)} = \delta_{ij}$. Furthermore, $X_i \succeq 0$ and 
$$
  \nbr{X_i} \le \tr(X_i) = \sum_{u=1}^m \gamma_u \phi^2_u(x_i) \le g(x_i, x_i) = \constg.
$$
Therefore, matrix Chernoff bound (e.g., \cite{Tropp12}) gives
$$
  \Pr\sbr{ \lambda_m(Y) \le  (1-\epsilon)\lambda_m(\EE Y)  } 
	\le m \exp\rbr{-\rbr{1-\epsilon}^2 \lambda_m(\EE Y)  / (2c_g)}.
$$
Choose $\epsilon = 1/2$ and use the facts that $G^{[m]} = AA^\top$, $Y = A^\top A$ and $\lambda_m(G^{[m]}) = \lambda_m(Y)$, we finish the proof.
\end{proof}

\begin{proof}[Proof of Lemma~\ref{lem:lowerG}]
By Weyl's theorem and the fact that $E^m$ is PSD,
 $$ 
   \eigg{m} \ge \lambda_m(G^{[m]})  + \lambda_m(E^m) \ge \lambda_m(G^{[m]}).
$$
Lemma~\ref{lem:lowerG} then follows from Lemma~\ref{lem:chernoff_gm}.
\end{proof}

\section{Bounding the difference between $\eigg{m}$ and $\eiggn{m}$: Proof of Lemma~\ref{lem:g_difference}} \label{app:diffG}
Using Weyl's theorem we have that 
\begin{align}
  \abr{ \lambda_m(G_n) - \lambda_m(G)} \le \nbr{G_n - G}.
\end{align}
We are going to give an upper bound on $\nbr{G_n - G}$:
\begin{align}
 \nbr{G_n - G} &=  \sup_{\nbr{\alpha} = 1} \sum_{i,j=1}^m \alpha_i \alpha_j (x_i^\top x_j)  E_{ij}, \\
\textnormal{where}~~  E_{i,j} & =  {\frac{1}{n}\sum_{k=1}^n \sigma^\prime(w_k^\top x_i) \sigma^\prime(w_k^\top x_j) - \EE_w[\sigma^\prime(w^\top x_i) \sigma^\prime(w^\top x_j)] },
\end{align}
and the first expectation is taken over $w$ uniformly on the sphere $\mathbb{S}^{d-1}$.

Our bound heavily relies on the inner products $\abr{\inner{x_i}{x_j}}$ for all $i \neq j$ being small enough. In the next lemma, we provide such a result for uniformly distributed  data.
\begin{lemma}[Tail bound for spherical distribution] \label{lem:sphericalangle} 
If $a$ and $b$ are independent vectors uniformly distributed over the unit sphere $\mathbb{S}^{d-1}$, then  there exists a constant $c > 0$, such that for any $u > 0$, 
$$
  \Pr\sbr{ \abr{ \inner{a}{b} } \ge  \frac{c u }{\sqrt{d}} } \le 2e^{- u^2}.
$$
\end{lemma} 

\begin{proof}
Note that both $a$ and $b$ are sub-gaussian random variables with sub-gaussian norm $c/\sqrt{d}$ where $c$ is some constant~\cite{Vershynin10}.

Denote $\EE_b\sbr{\cdot}$ the expectation over $b$.
We can rewrite the probability as
\begin{align}
\Pr\sbr{ \abr{ \inner{a}{b} } \ge  \frac{c u }{\sqrt{d}} } &\le \EE_b \Pr\sbr{\abr{ \inner{a}{b} } \ge  \frac{c u }{\sqrt{d}} \mid b} \nonumber\\
&\le \EE_b \cbr{2 \exp\rbr{-u^2}} = 2 \exp\rbr{-u^2}.
\end{align}

The last inequality uses the independence of $a$ and $b$ and $\nbr{\inner{a}{b}}_{\psi_2} \le \nbr{b}_2 \nbr{a}_{\psi_2}$ for a fixed $b$.
\end{proof}

Decomposing the sum into diagonal and off-diagonal terms gives us
\begin{align}
\nbr{G_n - G} & \le  \sup_{\nbr{\alpha} = 1}  \sum_{i\neq j}^m \alpha_i \alpha_j  \inner{x_i}{x_j}E_{ij} + \sum_{i=1}^m \alpha_i^2 E_{ii}  \\
& \le \sup_{\nbr{\alpha} = 1}  \sqrt{\sum_{i\neq j} \alpha_i^2 \alpha_j^2} \sqrt{\sum_{i\neq j} \inner{x_i}{x_j}^2 E_{ij}^2} + \max_i \abr{E_{ii}}.
\end{align}

Let $\Gcal$ denote the event that for all $i \neq j \in [m]$, $|\inner{x_i}{x_j}| \le O\rbr{ \frac{\log d}{\sqrt{d}} }$, then by Lemma~\ref{lem:sphericalangle} and the union bound
\begin{align}
\Pr\sbr{\neg \Gcal} & \le  2m^2 e^{- \log^2 d}.
\end{align}

Therefore, with probability at least $1 - 2m^2 e^{- \log^2 d}$, we have
\begin{align}
\nbr{G_n - G}  \le c\frac{\log d}{\sqrt{d}}\sqrt{ \sum_{i\neq j} E_{ij}^2 } +  \max_i \abr{E_{ii}}.
\end{align}

Note that 
\begin{align}
U(\cbr{x_1, \cdots, x_m}) = \frac{1}{m(m-1)}\sum_{i\neq j} E_{ij}^2
\end{align}
is a U-statistics. 

Suppose $\abr{E_{ij}} \le B$, according to the concentration inequality (Theorem 2 in \cite{PeeAntRal10}), we have with probability at least $1 - \delta$
 \begin{align}
 \sum_{i\neq j} E_{ij}^2 \le m(m-1)\EE_{\cbr{x_1,x_2}} E_{12}^2 + m(m-1)\rbr{\sqrt{\frac{4\Sigma^2}{m}\log\frac{1}{\delta}} + \frac{4B^2}{3m}\log \frac{1}{\delta}},
 \end{align}
 where $\Sigma^2 = \EE\sbr{E_{1,2}^4} - \EE\sbr{E_{1,2}^2}^2$ is the variance for the kernel in U-statistics.

Putting everything together, we have with probability at least $1 - \delta - 2m^2 e^{- \log^2 d}$
\begin{align}
\nbr{G_n - G} \le c\frac{\log d}{\sqrt{d}}\rbr{m\sqrt{\EE_{\cbr{x_1,x_2}} E_{12}^2} + m\rbr{\frac{4\Sigma^2}{m}\log\frac{1}{\delta}}^{1/4} + mB\sqrt{\frac{4}{3m}\log \frac{1}{\delta}}} + B
\end{align}

\section{Discrepancy of the weights}\label{app:discrepancy}
In this section, we relate the quantities $\EE_{\cbr{x_1,x_2}} E_{12}^2$ and $B$ to the discrepancies of the weights. Note that for ReLU, $\sigma'(w^\top x)$ does not depend on the norm of $w$, so we can focus on $w$ on the unit sphere.

Given a set of $n$ points $W = \cbr{w_i}_{i=1}^n$ on the unit sphere $\mathbb{S}^{d-1}$, the discrepancy of $W$ with respect to a measurable subset $S \subseteq \mathbb{S}^{d-1}$ is defined as
\begin{align}
  \text{dsp}(W, S) = \frac{1}{n} \abr{W \cap S} - A(S) 
\end{align}
where $A(S)$ is the normalized area of $S$ (\eg, the area of the whole sphere $A(\mathbb{S}^{d-1})$ is $1$). Let $\Scal$ denote the family of slices in $\mathbb{S}^{d-1}$
\begin{align}
  \Scal = \cbr{S_{xy} : x, y \in \mathbb{S}^{d-1}}, ~~\textnormal{where}~S_{xy} = \cbr{w \in \mathbb{S}^{d-1} : w^\top x \ge 0, w^\top y \ge 0}.
\end{align}
The $L_\infty$ discrepancy of $W$ with respect to $\Scal$ is 
\begin{align}
  L_\infty(W, \Scal) = \sup_{S \in \Scal} \text{dsp}(W, S),
\end{align}
and the $L_2$ discrepancy is 
\begin{align}
  L_2(W, \Scal) = \sqrt{ \EE_{x,y} \text{dsp}(W, S_{xy})^2}
\end{align}
where the expectation is taken over $x,y$ uniformly on the sphere. 
We use $L_\infty(W)$ and $L_2(W)$ as their shorthands.

For ReLU, by definition, we have
\begin{align}
  \EE E_{ij}^2 & = \rbr{ L_2(W) }^2, \\ 
	B & \le L_\infty(W), \\
	\Sigma^2 & \le \EE\sbr{E_{1,2}^4} \le \EE\sbr{E_{1,2}^2} \max_{x_1, x_2} \abr{E_{1,2}^2} \le \rbr{ L_\infty(W) L_2(W) }^2,
\end{align}
using the fact that $E_{ij} = \text{dsp}(W, S_{x_i x_j})$.

Therefore, the bound becomes
\begin{align}
\nbr{G_n - G} \le c\frac{\log d}{\sqrt{d}}\rbr{m L_2(W)  + \sqrt{ L_\infty(W) L_2(W) } m\rbr{\frac{4}{m}\log\frac{1}{\delta}}^{1/4}  + m L_\infty(W)\sqrt{\frac{4}{3m}\log \frac{1}{\delta}}} + L_\infty(W)
\end{align}

In the following subsections, we will discuss the discrepancies.

\subsection{Computing $L_2$ discrepancy for ReLU}
Note that the derivative of ReLU $\sigma'(w^\top x) = \II\sbr{w^\top x}$ does not depend on the norm of $w$. Without loss of generality, we can assume $\nbr{w}=1$ throughout this subsection.

\stolarsky*

\begin{proof}
Let $d(u,v) = \frac{\arccos\inner{u}{v}}{\pi}$. Let $S_{xy} = \cbr{w \in \mathbb{S}^{d-1} : w^\top x \ge 0,~w^\top y \ge 0}$. It is clear that (up to sets of measure zero)
\begin{align}
  A(S_{xy}) & = k(x,y) = \frac{1 - d(x,y)}{2}, \\
	\II\sbr{z \in S_{xy}} & = \frac{1}{4} \rbr{\sgn(x^\top z) + 1} \rbr{\sgn(y^\top z) + 1},
\end{align}
where $\II\sbr{\cdot}$ is the indicator function.
Then
\begin{align}
  \text{dsp}(W, S_{xy}) 
	& = \frac{1}{n}\sum_{k=1}^n \II\sbr{w_k \in S_{xy}} - A(S_{xy}) \\
	& = \frac{1}{n}\sum_{k=1}^n \frac{1}{4} \rbr{\sgn(x^\top w_k) + 1} \rbr{\sgn(y^\top w_k) + 1} - \frac{1 - d(x,y)}{2}.
\end{align}
Let $s_{xi}$ be a shorthand for $\sgn(x^\top w_i)$. Then we have
\begin{align}
  \rbr{L_2(W)}^2 
	& = \EE_{x,y} \rbr{\text{dsp}(W, S_{xy})}^2 \\
	& = \int_{\mathbb{S}^{d-1}} \int_{\mathbb{S}^{d-1}} \rbr{ \frac{1}{n}\sum_{k=1}^n \frac{1}{4} \rbr{\sgn(x^\top w_k) + 1} \rbr{\sgn(y^\top w_k) + 1} - \frac{1 - d(x,y)}{2} }^2 dA(x) dA(y) \\
	& = \frac{1}{n^2} \sum_{i,j=1}^n \int_{\mathbb{S}^{d-1}} \int_{\mathbb{S}^{d-1}} \frac{(s_{xi}+1)(s_{yi}+1)}{4} \frac{(s_{xi}+1)(s_{yi}+1)}{4} dA(x)dA(y) \\
	  & \quad - \frac{2}{n} \sum_{i=1}^n \int_{\mathbb{S}^{d-1}} \int_{\mathbb{S}^{d-1}} \frac{1-d(x,y)}{2}\frac{(s_{xi}+1)(s_{yi}+1)}{4} dA(x)dA(y) \\
		& \quad + \int_{\mathbb{S}^{d-1}} \int_{\mathbb{S}^{d-1}} \rbr{ \frac{1-d(x,y)}{2} }^2 dA(x)dA(y).
\end{align}
Consider the first term, which is equal to
\begin{align}
\frac{1}{n^2} \sum_{i,j=1}^n \rbr{\int_{\mathbb{S}^{d-1}} \frac{(s_{xi}+1)(s_{xj}+1)}{4} dA(x)}\rbr{\int_{\mathbb{S}^{d-1}} \frac{(s_{yi}+1)(s_{yj}+1)}{4} dA(y)}.
\end{align}
 By Lemma~\ref{lem:d_sign}, 
\begin{align}
\int_{\mathbb{S}^{d-1}} \frac{(s_{xi}+1)(s_{xj}+1)}{4} dA(x) = \frac{2 - 2 d(w_i, w_j)}{4},
\end{align}
so the first term is equal to 
\begin{align}
\frac{1}{n^2} \sum_{i,j=1}^n \rbr{\int_{\mathbb{S}^{d-1}} \frac{(s_{xi}+1)(s_{xj}+1)}{4} dA(x)}^2 = \frac{1}{n^2} \sum_{i,j=1}^n k(w_i,w_j)^2.
\end{align}
Now consider the second term. Note that the summand is invariant to $w_i$, so it can be replaced by an arbitrary $p \in \mathbb{S}^{d-1}$. The second term is then equal to 
\begin{align}
    & -2 \int_{\mathbb{S}^{d-1}} \int_{\mathbb{S}^{d-1}} \frac{1-d(x,y)}{2} \frac{(\sgn(x^\top p)+1)(\sgn(y^\top p)+1)}{4} dA(x)dA(y) \\
	= & -2 \int_{\mathbb{S}^{d-1}} \int_{\mathbb{S}^{d-1}} \frac{1-d(x,y)}{2} \II[p \in S_{xy} ] dA(x)dA(y) \\
	= & -2 \int_{\mathbb{S}^{d-1}} \int_{\mathbb{S}^{d-1}} \int_{\mathbb{S}^{d-1}} \frac{1-d(x,y)}{2} \II[p \in S_{xy} ] dA(x)dA(y) dA(p) \\	
	= & -2 \int_{\mathbb{S}^{d-1}} \int_{\mathbb{S}^{d-1}} \frac{1-d(x,y)}{2} \sbr{ \int_{\mathbb{S}^{d-1}} \II[p \in S_{xy} ] dA(p) } dA(x)dA(y) \\
	= & -2 \int_{\mathbb{S}^{d-1}} \int_{\mathbb{S}^{d-1}} \frac{1-d(x,y)}{2} \frac{2-2d(x,y)}{4}dA(x)dA(y)\\	
	= & -2 \int_{\mathbb{S}^{d-1}} \int_{\mathbb{S}^{d-1}} \rbr{\frac{1-d(x,y)}{2} }^2 dA(x)dA(y),
\end{align}
where the third step is by invariance to $p$ and the fourth step is by Lemma~\ref{lem:d_sign}.
The theorem then follows.
\end{proof}

Theorem~\ref{thm:stolarsky} lets us compute $L_2(W)$ for a fixed $W$. The next lemma gives a concrete bound for a special case where $W$ is uniformly distributed on the unit sphere.

\noindent
\textbf{Lemma~\ref{lem:l2discrepancy}}
{\it 
There exists a constant $c_g$, such that for any $0 < \delta < 1$, with probability at least $1-\delta$ over $W = \cbr{w_i}_{i=1}^n$ that are sampled from the unit sphere uniformly at random, 
\begin{align*}
  \rbr{L_2(W)}^2 \le c_g\rbr{\sqrt{\frac{\log d}{nd} \log \frac{1}{\delta}} + \frac{1}{n} \log \frac{1}{\delta}}.
\end{align*}
}

\begin{proof}
By Theorem~\ref{thm:stolarsky}, we have 
\begin{align}
	\rbr{L_2(W)}^2 & = \frac{1}{4 n^2} \sum_{i,j = 1}^n \rbr{\frac{1}{2} - d(w_i,w_j)}^2  - \frac{1}{4}\int_{\mathbb{S}^{d-1}} \int_{\mathbb{S}^{d-1}} \rbr{\frac{1}{2} - d(u,v)}^2 dA(u) dA(v) + \frac{1}{4n^2} \sum_{i,j =1}^n \rbr{\frac{1}{2} - d(w_i,w_j)}.    
\end{align}

First consider $T_1 = \frac{1}{n^2} \sum_{i,j = 1}^n \rbr{\frac{1}{2} - d(w_i,w_j)}^2  - \mu$ where $\mu = \int_{\mathbb{S}^{d-1}} \int_{\mathbb{S}^{d-1}} \rbr{\frac{1}{2} - d(u,v)}^2 dA(u) dA(v)$. 
Rewrite $T_1 = \frac{1}{4n} + \frac{n-1}{n}U(W) - \mu$ where $U(W) = \frac{1}{n(n-1)}\sum_{i\neq j}  \rbr{\frac{1}{2} - d(w_i,w_j)}^2 $ is a U-statistics.  We upper bound $U(W)$ by using Bernstein's inequality when $W$ is uniform over the sphere.

By Taylor expansion, we have 
\[
  \frac{1}{2} - d(u,v) = x/\pi + x^3/6\pi + O(x^5), \textnormal{~where~} x = u^\top v. 
\]
Then let $\Gcal$ denote the event that $|x| = |u^\top v| \le c\sqrt{\log d/d}$ for a sufficient large constant $c>0$, so that by Lemma~\ref{lem:sphericalangle}, $\Pr[\neg\Gcal] \le O(1/d^4)$. 
Then 
\begin{align}
\EE \sbr{U(W)} &=   \mu = \EE\sbr{\rbr{ x/\pi + x^3/6\pi + O(x^5) }^2} \\
	& = \EE[x^2/\pi^2 + x^4/6\pi^2 + O(x^6)] \\
	& \le \EE\sbr{ [x^2/\pi^2 + x^4/6\pi^2 + O(x^6)] \mid \Gcal} + \Pr[\neg \Gcal] \max_{u,v} \sbr{\frac{1}{2} - d(u,v)}^2 \\
	& = O\rbr{\frac{\log d}{d}},
\end{align}
and thus
\begin{align}
  \textnormal{Var}\sbr{ U(W) } & = \EE\cbr{ \sbr{\rbr{x/\pi + x^3/6\pi + O(x^5)}^2 - \mu}^2 } \\
	& = \EE\cbr{ \sbr{x^2/\pi^2 + x^4/6\pi^2 + O(x^6) - \mu}^2 } \\
	& \le \EE\cbr{ \sbr{x^2/\pi^2 + x^4/6\pi^2 + O(x^6) - \mu}^2 \mid \Gcal} + \Pr[\neg\Gcal] \max_{u,v} \sbr{\rbr{\frac{1}{2} - d(u,v)}^2 - u}^2 \\
	& = O\rbr{\frac{\log^2 d}{d^2}}.
\end{align}
Then by Berstein's inequality, we have with probability at least $1-\delta$ over the $W$ uniformly on the sphere, 
\begin{align}
  \abr{T_1} \le O\rbr{\frac{\log d}{d} \sqrt{\frac{1}{n} \log \frac{1}{\delta}} + \frac{1}{n} \log \frac{1}{\delta}}. 
\end{align}

A similar argument holds for $T_2 = \frac{1}{n^2} \sum_{i,j =1}^n \rbr{\frac{1}{2} - d(w_i,w_j)}$. Note that 
\begin{align}
  \textnormal{Var}\cbr{\rbr{\frac{1}{2} - d(u,v)} } = \mu = O\rbr{\frac{\log d}{d}}.
\end{align}
We have that with probability at least $1-\delta$ over the $W$ uniform from the sphere, 
\begin{align}
  \abr{T_2} \le O\rbr{\sqrt{\frac{\log d}{nd} \log \frac{1}{\delta}} + \frac{1}{n} \log \frac{1}{\delta}}. 
\end{align}
This completes the proof.
\end{proof}

Below are some technical lemmas that are used in the analysis. 

\begin{lemma} \label{lem:d_sign}
\begin{align}
	\int_{\mathbb{S}^{d-1}} d(x,y) d A(x) & = \frac{1}{2}, \forall y \in \mathbb{S}^{d-1}, \\
	\int_{\mathbb{S}^{d-1}} \sgn(x^\top y) d A(x) & = 0, \forall y \in \mathbb{S}^{d-1}, \\
	\int_{\mathbb{S}^{d-1}} \sgn(x^\top z) \sgn(y^\top z) d A(z) & = 1 - 2d(x,y), \forall x, y \in \mathbb{S}^{d-1}.
\end{align}
\end{lemma}
\begin{proof}
The first two are straightforward. The third is implicit in the proof of Theorem 1.21 in~\cite{BilLac15}. 
\end{proof}

%
%

\section{The spectrum of $\gamma_m$} \label{app:spectrum}
The spectrum of the kernel matrix $g(x, y) = \rbr{\frac{1}{2} - \frac{\arccos\inner{x}{y}}{2\pi}}  \inner{x}{y}$ is determined by the spherical decomposition coefficients.

We need $\gamma_m$ to decrease slower than $O(1/\sqrt{m})$ within a reasonable range, such as $m \le1000000$.

Although the kernel associated with ReLU decreases faster than the desired rate, we can choose from a family of such arccos kernels such that the spectrum decays slower than $1/\sqrt{m}$.

Such arccos kernels are defined as
\begin{align}
k_n(x, y) = J_n(\theta) / \pi
\end{align}
where
\begin{align}
J_n(\theta) = (-1)^n (\sin \theta)^{2n+1}\rbr{\frac{1}{\sin\theta}\frac{\partial}{\partial \theta}}^n \rbr{\frac{\pi-\theta}{\sin\theta}}
\end{align}

Larger $n$ corresponds to more nonlinear activation functions and leads to slower decaying spectrum.
\begin{align}
J_0(\theta) &= \pi - \theta \nonumber\\
J_1(\theta) &= \sin\theta + (\pi -\theta) \cos\theta \nonumber\\
J_2(\theta) &= 3\sin\theta\cos\theta + (\pi -\theta)(1+ 2\cos^2\theta)  \nonumber\\
4 J_3(\theta) &= 27\sin\theta +11(3\sin\theta\cos^2\theta - \sin^3\theta) \nonumber\\
		&+ (\pi-\theta)(54\cos\theta + 6(\cos^3\theta - 3\sin^2\theta\cos\theta)) \nonumber\\
4 J_4(\theta) &= (\pi-\theta)(216+192\cos(2\theta)+ 12\cos(4\theta))  \nonumber\\
		     &+ 160\sin(2\theta) + 25\sin(4\theta)  \nonumber\\
8 J_5(\theta) &= (\pi-\theta) (6000\cos\theta + 1500\cos(3\theta) + 60\cos(5\theta))	\nonumber\\
			&+ 2000\sin\theta + 1625\sin(3\theta) + 137\sin(5\theta)	\nonumber\\
\frac{8}{9} J_6(\theta) &= (\pi - \theta)(4000+4500\cos(2\theta)+720\cos(4\theta)+20\cos(6\theta))	\nonumber\\
		  &+ 2625\sin(2\theta) + 924\sin(4\theta) + 49\sin(6\theta)	\nonumber\\
\frac{16}{9} J_7(\theta) &= (\pi - \theta)(171500\cos(\theta) + 61740\cos(3\theta) + 6860\cos(5\theta)  + 140\cos(7\theta))	\nonumber\\
		  &+ 42875\sin(\theta) + 48363\sin(3\theta) + 9947\sin(5\theta) + 363\sin(7\theta)
\end{align}

Higher orders of $J_n(\theta)$ seems to be extremely complicated.

Although there is no analytical solution to the spectrum, we can compute them numerically. 

Figure~\ref{fig:spectrum} illustrates the spectra of several arccos kernels compared to $O(1/m)$ and $O(1/\sqrt{m})$.

\section{Rademacher complexity and final error bounds: Proof of Theorem~\ref{thm:error} and Theorem~\ref{thm:main} }\label{sec:rademacher}


We apply the argument in~\cite{BarMen02} to our setting to get Lemma~\ref{lem:rad}. Combining it with Theorem~\ref{thm:singular_value} leads to Theorem~\ref{thm:error}. Further combining it with Lemma~\ref{lem:l2discrepancy} leads to Theorem~\ref{thm:main}.
\begin{lemma} \label{lem:rad}
Suppose the data are bounded: $\abr{y} \le Y$ and $\nbr{x}_2\le 1$.
Let 
$$ 
  \Fcal = \cbr{f(x) = \sum_{k=1}^n v_k \sigma(w_k^\top x) :  v_k \in \cbr{-1,+1},~\sum_k\nbr{w_k}_2 \le C_W}. 
$$
Then with probability $\ge 1-\delta$, for any $f \in \Fcal$, 
\begin{align}
\frac{1}{2}\EE(y-f(x))^2 \le \frac{1}{2m}\sum_{l=1}^m (y_l-f(x_l))^2 + \frac{2(Y+C_W) C_W}{\sqrt{m}} + (Y^2 + C_W^2)\sqrt{\frac{\log\frac{1}{\delta}}{2m}}.
\end{align}
\end{lemma}

\begin{proof}
For a sample $S = ((x_1,y_1), \cdots, (x_m,y_m))$, and a loss function $\ell(y,x) = \frac{1}{2}(y-f(x))^2$, we denote
$\hat{\EE}_S[\ell]$ as the empirical average $\hat{\EE}_S[\ell] = \frac{1}{m}\sum_{l=1}^m \ell(y_l,x_l)$.  

Define
\begin{align}
\Phi(S) = \sup_{\ell \in \Lcal} \EE[\ell] - \hat{\EE}_S[\ell]
\end{align}
where $\Lcal$ is the set of loss functions
$$
 \Lcal = \cbr{\ell(y,x) = \frac{1}{2}(y-f(x))^2 : f \in \Fcal}.
$$

Let $S$ and $S'$ be two datasets that differ by exactly one data point $(x_i, y_i)$ and $(x'_i, y'_i)$.  Then we have a bound on the difference of loss functions. Since $\nbr{x}_2 \le 1$ and $\sum_k\nbr{w_k}_2 \le C_W$, we have $\abr{f} \le C_W$. Thus
\begin{align}
\abr{ \ell(y, f(x)) - \ell(y', f(x')) } & \le \frac{1}{2}\max\cbr{ (y - f(x))^2 , (y'- f(x'))^2} \le Y^2 + C_W^2.
\end{align}

This leads to an upper bound
\begin{align}
\Phi(S) - \Phi(S')  &\le \sup_{\ell \in \Lcal} \hat{\EE}_S[\ell] - \hat{\EE}_{S'}[\ell] \nonumber\\
&= \sup_{\ell \in \Lcal} \frac{\ell(y_i,f(x_i)) - \ell(y'_i,f(x'_i))}{m} \le \frac{Y^2 + C_W^2}{m}.
\end{align}

Similarly, we can get the other side of the inequality and have $\abr{\Phi(S) - \Phi(S')} \le \frac{Y^2 + C_W^2}{m}$.

From McDiamids' inequality, with probability at least $1 -\delta$ we get 
\begin{align}
\Phi(S) \le \EE_{S}\Phi(S) + (Y^2 + C_W^2)\sqrt{\frac{\log\frac{1}{\delta}}{2m}} .
\end{align}

The first term on the right-hand side can be bounded by Rademacher complexity as shown in the book Foundations of Machine Learning (3.13).
In the end, we have the bound
\begin{align}
\frac{1}{2}\EE(y-f(x))^2 \le \frac{1}{2m}\sum_{l=1}^m (y_l-f(x_l))^2 + 2\Rcal_m(\Lcal) + (Y^2 + C_W^2)\sqrt{\frac{\log\frac{1}{\delta}}{2m}}
\end{align}
where $\Rcal_m(\Lcal)$ is the Rademacher complexity of the function class $\Lcal$.

We can find the Rademacher complexity by using composition rules. The Rademacher complexity of linear functions $\cbr{w^\top x : \nbr{w}_2\le b_W, \nbr{x}_2\le1}$ is $b_W / \sqrt{m}$, where $m$ is the number of data points. If a function $\phi$ is $L$-Lipschitz, then for any function class $\Hcal$, we have $\Rcal(\phi \circ \Hcal )\le L\Rcal(\Hcal)$. In addition, we also have $\Rcal(cH) = \abr{c}\Rcal(H)$ and $\Rcal(\sum_k F_k) \le \sum_k\Rcal(F_k)$.

So for the function class $\Fcal$ that describes a neural network, we have
\begin{align}
\Rcal_m(\Fcal) \le \frac{C_W}{\sqrt{m}}.
\end{align}
It is derived by using the fact that $\sigma'(\cdot)$ is 1-Lipschitz and $\sum_k \nbr{w_k}_2 \le C_W$.

Finally composing on the loss function we get
\begin{align}
\Rcal_m(\Lcal) \le  \frac{(Y+C_W) C_W}{\sqrt{m}},
\end{align}
using the fact that the ground truth in the loss should be bounded by $Y$ and the function bounded by $C_W$, thus the Lipschitz constant of the loss function is bounded by $Y+C_W$.
\end{proof}

\end{document}